\documentclass[12pt]{article}
\usepackage[margin=1in]{geometry}

\usepackage{
        appendix,
	amsmath,
	amsthm,
	amssymb,
	mathrsfs,
	MnSymbol,
	empheq,
	graphics,
	graphicx,
	enumitem,
	bbm,
	xcolor,
	tikz,
	relsize,
	float,
	verbatim,
	hyperref,
	color,
	courier,
	listings,
	datetime,
	lastpage,
	titling,
	caption,
	subcaption,
	booktabs,
        natbib}

\usepackage[english]{babel}
\usepackage[utf8]{inputenc}
\usepackage[T1]{fontenc}
\usepackage[framemethod=default]{mdframed}
\usepackage[super]{nth}
\usepackage[parfill]{parskip}

\hypersetup{
	colorlinks,
	citecolor=black,
	filecolor=black,
	linkcolor=black,
	urlcolor=blue
}

\newtheorem{lem}{Lemma}

\newcommand{\ba}{\[\begin{aligned}}
\newcommand{\ea}{\end{aligned}\]}

\begin{document}

\title{Reward Shaping via Diffusion Process in Reinforcement Learning} 
\author{Peeyush Kumar}
\date{}
\maketitle

\begin{abstract}
    Reinforcement Learning (RL) models have continually evolved to navigate the exploration - exploitation trade-off in uncertain Markov Decision Processes (MDPs). In this study, I leverage the principles of stochastic thermodynamics and system dynamics to explore reward shaping via diffusion processes. This provides an elegant framework as a way to think about exploration-exploitation trade-off. This article sheds light on  relationships between information entropy, stochastic system dynamics, and their influences on entropy production. This exploration allows us to construct a dual-pronged framework that can be interpreted as either a maximum entropy program for deriving efficient policies or a modified cost optimization program accounting for informational costs and benefits. This work presents a novel perspective on the physical nature of information and its implications for online learning in MDPs, consequently providing a better understanding of information-oriented formulations in RL.
\end{abstract}


\section{Introduction}\label{sec:thermo}
In this article, I take inspiration from stochastic thermodynamics to derive a problem formulation for online learning in uncertain MDPs while grounded in system dynamics. The system balances the diffusion process with drif dynamics as a way to formulate the exploration-exploitation trade-off.

To this effect, I make an explicit link between the information entropy and the stochastic dynamics of a system coupled to an environment.  
I analyze various sources of entropy production: due to the decision-maker's uncertainty about the system-environment interaction characteristics; due to the stochastic nature of system dynamics; and the interaction of the decision maker's knowledge with system dynamics. This analysis provides a framework that can be formulated either as a maximum entropy program to derive efficient policies that balance the exploration and exploitation trade-off, or as a modified cost optimization program that includes informational costs and benefits. 

\subsection{Background}\label{sec:idpsintro}
Markov decision processes (MDPs) are perhaps the most widely studied models 
of sequential decision problems under uncertainty \citep{PUT94}. In this article, an MDP is described by the tuple $\mathcal M=(S,A,T, R,N)$. Here, $S$ is a finite set of states; $A$ is a finite set of actions; 
$T$ denotes a transition probability function of the form $T(s'|s,a)$, for $s,s'\in S$ 
and $a\in A$; $R$ denotes a reward function of the form $R(s'|s,a)$, for $s, s'\in S$ and $a\in A$; and $N$ 
denotes a finite planning horizon. This MDP models the following 
time-invariant, finite-horizon, sequential decision-making problem under 
uncertainty. A decision-maker observes the state $s_t\in S$ of a system at the 
beginning of time-slot $t\in \{1,2,\ldots,N\}$ and then chooses an action $a_t\in A$. The system 
then stochastically evolves to a state $s_{t+1}\in S$ by the beginning of slot $t+1$ with probability $T(s_{t+1}|s_t,a_t)$. As a result of this transition, the 
decision-maker collects a reward $R(s_{t+1}|s_t,a_t)$. This process of state 
observation, action selection, state evolution, and reward collection repeats until the end of slot $N$. A policy trajectory $\pi=(\pi_1,\pi_2,\ldots,\pi_N)$ is a decision-rule that assigns actions $\pi_t(s_t)\in 
A$ to states $s_t\in S$, for $t=1,2,\ldots,N$. Note that the set $\mathcal P$ of 
such policy trajectories is finite. The decision-maker's 
objective is to find a policy trajectory $\pi=(\pi_1,\pi_2,\ldots,\pi_N)\in\mathcal P$ 
that maximizes the expected reward 
\begin{equation*}
J_\pi(s_1)=E\Bigg[\sum\limits_{t=1}^N R(s_{t+1}|s_t,\pi_t(s_t))\Bigg].
\end{equation*}
It is assumed for simplicity of notation that no terminal reward is earned at the 
end of slot $N$. 

The transition probability function is often unknown to the decision-maker at the 
outset. This calls for online learning of transition 
probabilities while the system evolves. For instance, in 
medical treatment planning, a doctor might not know the uncertain 
dose-response 
function of an individual at the beginning of a treatment course, but may want to 
adaptively make drug selection and dosing decisions over the treatment course \citep{kotas2016response}. 
Similarly, a seller conducting a sequence of auctions may not know the bidder 
demand 
and willingness-to-pay distributions, but must adaptively make auction-design 
decision such as the minimum bid in each auction \citep{ghate2015optimal}. Such 
problems fall under the 
broad framework of MDPs under imperfect information, and can be seen as 
Bayesian adaptive MDPs (BAMDPs) or partially observable MDPs (POMDPs) in 
some cases \citep{bertsekas,dreyfus,krishnamurthy,kumar85,kumarbook}.

The challenge in any Bayesian learning approach is that there is no clear consensus on the actual problem that needs to be solved. Generally, we want to 
find a policy that \textit{maximizes} cumulative reward while learning with 
uncertain or partial information. BAMDPs provide a classic formulation of this problem. But this formulation does not take into 
account the cost of information gain, and hence intuitively one can 
find better policies that leverage information gain while learning. The information 
theoretic methods developed so far rely on the heuristic idea of information ratio, 
which, I believe, is somewhat ad-hoc. In addition, this ratio does not give a strong 
insight into the global problem that is being solved. I find a relation  between the optimization of reward and the cost of information that is embedded in the dynamics of the system and its interaction with the environment.

\begin{figure}
\centering
  \includegraphics[width=0.5\textwidth]{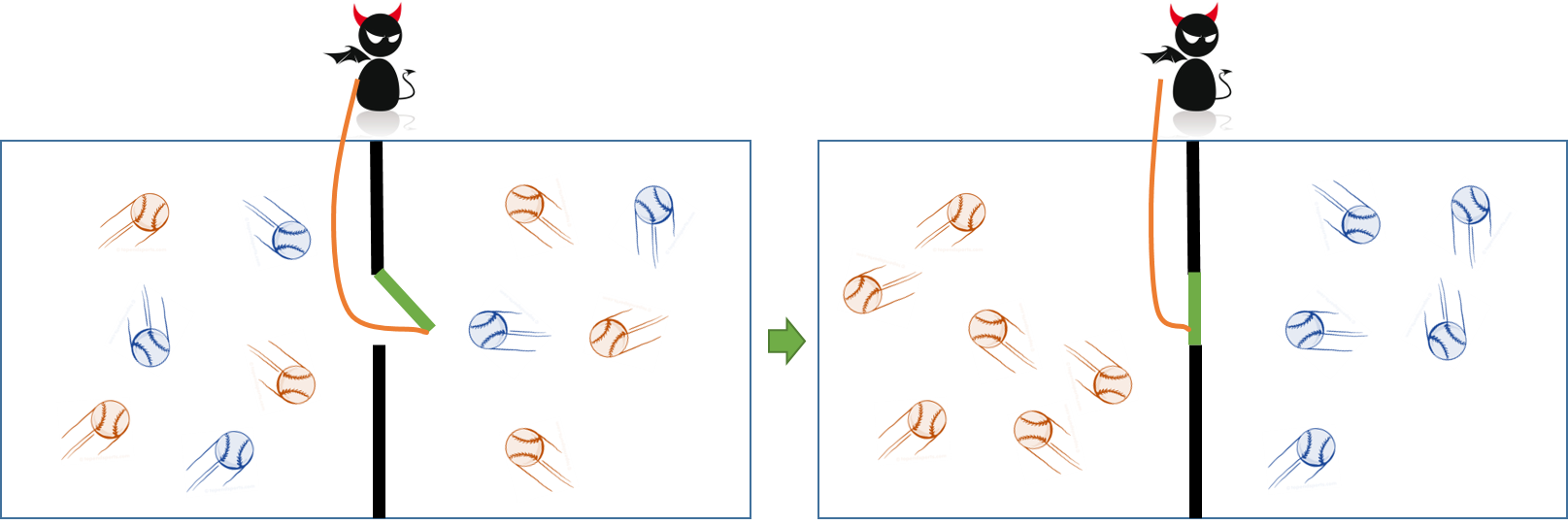}
  \caption{Maxwell's demon}
  \label{fig:md}
\end{figure}
\section{Physical nature of information}\label{informationphysical}
In order to motivate the idea of the physical nature of information, I dive into the
role of information in thermodynamics of gases. To guide the reader, a natural 
connection is to interpret particle configurations in gas systems as sample 
trajectories in stochastic system. Physicist Ludwig Boltzmann  showed that with time, a system evolves towards 
lower states of energy, where the energy dispersed increases the entropy of the 
system due to the nature of statistics \citep{boltzmann1974second}. As \cite{websiteThermo} commented, ``
\textit{There are many ways for energy to be spread among the particles in a 
system than concentrated in a few, so as particles move around and interact, 
they naturally tend toward states in which their energy is increasingly shared. 
This has been classically understood as the second law of thermodynamics. But 
Maxwell's letter \citep{maxwell1921theory} described a thought experiment in 
which an enlightened being, called Maxwell's demon (Figure~\ref{fig:md}), uses its knowledge to lower 
entropy and violate the second law. The demon knows the positions and 
velocities of every molecule in a container of gas. By partitioning the container 
and opening and closing a small door between the two chambers, the demon lets 
only fast-moving molecules enter one side, while allowing only slow molecules to 
go the other way. The demon's actions divide the gas into hot and cold, 
concentrating its energy and lowering its overall entropy. The once useless gas 
can now be put to work. This thought experiment lead to questions on how a law 
of nature could depend on one's knowledge of the positions and velocities of 
molecules. [This implies that second law of thermodynamics require a 
reinterpretation to include the subjective nature of information.]
Charles Bennett \citep{bennett1987demons}, building on work by Leo Szilard 
\citep{szilard1976entropy} and Rolf Landauer \citep{landauer1961irreversibility}, 
resolved the paradox by formally linking thermodynamics to the science of 
information. Bennett argued that the demon's knowledge is stored in its memory, 
and memory has to be erased, which takes work. \citep{landauer1961irreversibility} calculated that at room temperature, it takes at 
least 2.9 zeptojoules of energy for a computer to erase one bit of stored 
information.) In other words, as the demon organizes the gas into hot and cold 
and lowers the gas's entropy, its brain burns energy and generates more than 
enough entropy to compensate. The overall entropy of the gas-demon system 
increases, satisfying the second law of thermodynamics. These findings revealed 
that, as Landauer put it, ``Information is 
physical'' \citep{landauer1991information}.  More information implies that more 
work can be extracted. Maxwell's demon can wring work out of a 
single-temperature gas because it has far more information than the average user.}''

This interaction of entropy and dynamics, capture by the second law, creates a 
strong foundation to analyze stochastic systems. There is a natural equivalence 
between stochastic thermodynamics and stochastic control theory. Any decision 
process can be modeled as a classic control problem. Generally, the quantities 
which are of interest are averaged over trajectories of the system rather than 
sample path behaviors. Thermodynamics has provided an intuitive framework 
and solution about averaged entities on stochastic systems. I study 
this equivalence and bridge gaps in the existing literature on learning in MDPs. I develop an equivalent thermodynamic system and apply an  
information theoretic framework to 
find a formulation of the learning problem to compute good policies. 

There has been some work in the literature to bridge this gap between control 
theory and stochastic thermodynamics. \cite{brockett1979stochastic} studies 
second law of thermodynamics from the point of view stochastic control theory. 
They compute a criterion which, when satisfied, permits one to assign a 
temperature to a stochastic system in a way that Carnot cycles become the 
optimal trajectories of optimal control problems. \cite{propp1985thermodynamic} 
also studied the connection between thermodynamic and Markovian systems. 
There, an input-output framework for thermodynamics was proposed, which 
allowed to introduce the notion of states, controls and response, thus drawing a 
connection between the two fields. There has also been a recent surge in 
understanding the field of stochastic thermodynamics to study Markovian 
processes at the trajectory level using statistical quantities 
\citep{seifert2011stochastic,aurell2012refined}. \cite{saridis1988entropy} proposed 
a formulation that gives a generalized energy interpretation to the optimal control 
problem. This framework provides compatibility between the control problem and 
the information theoretic methodology for the
intelligent control system using entropy as the common measure. A reformulation 
of the optimal control problem is based on the idea of expressing the design of 
the desirable control by the uncertainty of selecting a control law that minimized a 
given performance index. 

\section{Clairvoyant MDP: an information theoretic perspective}\label{sec:clavmdp}
Consider the Bellman's equation for MDP $M = \{S,A,T,R,N\}$.
\begin{equation}
V^*(s) = \min_a \sum_{s'}T(s'|s,a)[R(s'|s,a) + V^*(s')] .
\end{equation}
I consider an alternate formulation to this classical MDP, with a small loss of generality. \cite{todorov2009efficient} proposed a linear problem where actions that are considered symbolic in the above formulation are replaced through making decisions over transition distributions. Therefore, the decision maker specifies a control dynamics distribution $a(s'|s) = T(s'|s,a)$. This allows us to write an equivalent reward form as

$$
q(s,a) = \ell(s) + \underset{s'\sim a(\cdot|s) }{E} \ln\left(\frac{a(s'|s)}{p(s'|s)}\right),
$$
where the state cost $\ell(s)$ is an arbitrary function encoding how
undesirable different states are and $p(s'|s) $ is an arbitrary transition distribution. Using this construction the Bellman's equation can be rewritten as:
\begin{equation}\label{eq:valmod}
V^*(s) = \min_{a}\left(\ell(s) + \underset{s'\sim a(\cdot|s)}{E}\left[\ln\frac{a(s'|s)}{p(s'|s)} + V^*(s')\right]\right) .
\end{equation}
Now, I define the quantity $G(s) = \underset{s'\sim p(\cdot|s)}{E} exp(-V^*(s'))$. Therefore, through some algebraic manipulation, I get
$$
 \underset{s'\sim a(\cdot|s) }{E}\left[\ln\frac{a(s'|s)}{p(s'|s)} + V^*(s')\right] = -\ln (G(s)) + \mathbb{KL}\left(a(\cdot|s) || \frac{p(\cdot|s)\exp(-V^*(\cdot))}{G(s)}\right),
$$
which gives
\begin{equation}
V^*(s) = \min_{a}\left[\ell(s) -\ln (G(s)) + \mathbb{KL}\left(a(\cdot|s) || \frac{p(\cdot|s)\exp(V^*(\cdot))}{G(s)}\right)\right].
\end{equation}
An interesting observation is that the right hand side of the above function is minimized when the KL divergence is 0, which gives the optimality condition as
\begin{align}
 a^*(s'|s) &= \frac{p(s'|s)\exp(-V^*(s'))}{G(s)}\\
&= \frac{p(s'|s)\exp(-V^*(s'))}{\sum_{s'}p(s'|s) exp(-V^*(s'))}\label{eq:opta}
\end{align}
Now consider the following Lemma \citep{theodorou2012relative,theodorou2015nonlinear}.
\begin{lem}\label{lem:base}
Consider distributions $\mathbb{A}$ and $\mathbb{P}$ defined on the same probability space with sample set $\Omega$, such that  $\mathbb{A}$  is absolutely continuous with respect to $\mathbb{P}$, and $Q:\Omega \mapsto \mathbb{R}$ is a measurable function, then the following inequality holds
$$
\frac{1}{\rho} \ln \left( \underset{\mathbb{P}}{E}\left[e^{\rho Q(s)}\right]\right) \leq \underset{\mathbb{A}}{E}[ Q(s) +|\rho|^{-1} \mathbb{KL}(\mathbb{A}||\mathbb{P})   ],
$$
where $\rho \in \mathbb{R}^-$.
\end{lem}
\begin{proof}
The proof is reproduced here for completeness. It is a straightforward derivation from Jensen's inequality. Consider,
\begin{align*}\label{lem:febase}
\ln \left( \underset{\mathbb{P}}{E}\left[e^{\rho Q(s)}\right]\right) &= \ln \sum_{s}p(s)e^{[\rho Q(s)]}\\
&= \ln \left[\sum_{s}a(s)\frac{p(s)}{a(s)} \exp{(\rho Q(s))}\right]\\
&\stackrel{a}{\geq}  \sum_{s} a(s) \ln\left[\frac{p(s)}{a(s)}  \exp{(\rho Q(s))}\right]\\
&=\rho \underset{\mathbb{A}}{E} [ Q(s)] +  \sum_{s} a(s)\ln\frac{p(s)}{a(s)}\\
&= \rho\left( \underset{\mathbb{A}}{E} [ Q(s)] - \rho^{-1}\mathbb{KL}(\mathbb{A||P})\right),
\end{align*}

where inequality ``a'' follows from Jensen's inequality and concavity of the $\ln$ function. Now dividing both sides by $\rho \in \mathbb{R}^-$ gives the required inequality.\\
\end{proof}

Next, consider Equation~(\ref{eq:valmod}), where I substitute $Q(s') = \ell(s) +V^*(s')$. Now using Lemma~\ref{lem:base} with $\rho = -1$, $\mathbb{P} = p(s'|s)$, $\mathbb{A}=a(s'|s)$, I get
\begin{align*}
-\ln \left( \underset{s'\sim p(\cdot|s)}{E}\left[e^{-\ell(s) - V^*(s')}\right]\right) \leq \sum_{s'}a(s'|s)\left[\ell(s) + V^*(s') + \ln\frac{a(s)}{p(s)}\right],
\end{align*}
which implies
$$
-\ln \left( \underset{s'\sim p(\cdot|s)}{E}\left[e^{-\ell(s) - V^*(s')}\right]\right) = \min_a \sum_{s'}a(s'|s)\left[\ell(s) + V^*(s') + \ln\frac{a(s)}{p(s)}\right].
$$
The right hand side of the above equation is the right hand side of the Bellman equation in Equation~(\ref{eq:valmod}). Therefore
$$
V^*(s) = -\ln \left( \underset{s'\sim p(\cdot|s)}{E}\left[e^{-\ell(s) - V^*(s')}\right]\right).
$$

This framework can also be used as an estimation framework where instead of minimizing the expected cumulative cost, the decision maker can maximize the KL divergence for a required performance. Therefore, the optimization problem in Equation~(\ref{eq:valmod}) becomes
$$
\min_{\mathbb{A}} \mathbb{KL(A||P)},
$$
subject to
$$
\sum_{s'} a(s'|s) = 1,
$$
$$
V(s) = K,
$$
where $K$ is the required performance.
In the interest of providing interesting connection, I consider continuous optimization. Using the Lagrangian method, the optimization program reduces to
\begin{align*}
\mathcal{L} &= \mathbb{KL(A||P)} +\mu (V(s) - K) +\lambda(\int_{s'} a(s'|s) ds' - 1) \\
&= -\int_{s'} a(s'|s)\left(ln \frac{a(s'|s)}{p(s'|s)} + \mu V(s) + \lambda\right)ds' + \mu K + \lambda
\end{align*}
Now, maximizing with respect to $a(s'|s)$ gives
$$
ln \frac{a^*(s'|s)}{p(s'|s)} + \mu V(s) + \lambda = 0,
$$
which gives
$$
a^*(s'|s) = \int exp(-mu V(s) - \lambda)p(s'|s) ds'.
$$
Substituting in the first constraint for $\int a(s'|s) ds = 1$, gives
$$
\lambda =\ln\int p(s'|s)\exp(-\mu V(s))ds'.  
$$
Substituting $\lambda$ back and discretizing gives the optimal solution for $a^*$ for a given level of performance. In the case where $K = V^*(s)$, this solution gives the optimal $a^*$ as in Equation~\ref{eq:opta}. This result is very similar to the one derived using HJB principle in the classic paper by \cite{saridis1988entropy}. For the reader's convenience, I recall that result in Appendix~\ref{app:maxent}.

\section{Thermodynamics of information}
This section provides a brief introduction on the relationship between information and thermodynamics.
We consider a system $M$ (such as a gas in a container) that is connected to external reservoirs and other systems. Suppose the microstate of the system (for example, the coordinates and momentum of particles of the gas) is given by $x$, and suppose that the information gained as a result of measurement is denoted by $m$. This measurement is what helps to prepare the state of the system. Let us denote a generic statistical state of the system with $\rho(x)$ (for example, the distribution of coordinate states and momentum of the gas molecules). I assume that in state $\rho(x)$ the system is in statistical equilibrium. Now after making the measurement, the new state of the system in $\rho(x|m)$, which in general is out of equilibrium. For example, in the context of the Sczilard's engine described in Section~\ref{informationphysical}, after measurement the statistical state is confined to either the left or right half of the box. Information drives
the system away from equilibrium. The thermodynamics of information allows us to reason about this scenario by associating an equivalent energy cost, thus justifying this movement from equilibrium to a non-equilibrium state.

The most obvious entity that relates statistical entities to distributions is the entropy of the system. In this case, the non equilibrium entropy is defined using a scaled version of the Shannon Entropy as
$$
S(\rho) = -\sum_x \rho(x) \ln\rho(x) = H(X),
$$
where $H(X)$ is the Shannon entropy.
At equilibrium this entropy coincides with the cannonical entropy

$$
\rho(x) = \exp^{-\beta E(x)}/Z,
$$
where $E(x)$ is the Hamiltonian of the system, and $Z$ is the partition function, and $\beta$ is the inverse temperature. Using this we recover the thermodynamic relationship between Free energy $\mathcal{F}(\rho) = -\beta^{-1} \ln Z$, and internal Energy $E = E[H]$ and Entropy: $\mathcal{F} = E - \beta^{-1}S$. The free energy is interpreted as the amount of useful energy that can be used to extract work, taking in account all entropy related costs. The classical second law of thermodynamics for non equilibrium system, therefore, can be written as
\begin{equation}\label{eq:slt}
\Delta S \geq 0 \implies W - \Delta \mathcal{F} \geq 0,
\end{equation}
where $W$ is the average work done on the system.

The rest of the section evaluates the change in non-equilibrium free energy due to a measurement $M$. For this purpose the corresponding information gain is defined as
$$
I(X;M) = H(X) - H(X|M).
$$
Now, in the event that an external system changes the system parameter after an observation is made, results in work extracted from the system. The refined second law of thermodynamics then becomes
\begin{equation}\label{eq:sltref}
W - \Delta \mathcal{F} \geq -\beta^{-1} I(X;M)
\end{equation}

An interesting observation is that ultimately, the information used to extract work during feedback
was supplied as work by the external system during the measurement process.

\section{Markovian systems and second law of thermodynamics}\label{sec:msslt}
Now let's consider the second law of thermodynamics $W \geq \Delta \mathcal{F} $ without feedback (Equation~\ref{eq:slt}), and compare it with the Lemma~\ref{lem:base}
$$
\frac{1}{\rho} \ln \left( \underset{\mathbb{P}}{E}\left[e^{\rho Q(s)}\right]\right) \leq \underset{\mathbb{A}}{E}[ Q(s) +|\rho|^{-1} \mathbb{KL}(\mathbb{A}||\mathbb{P})   ].
$$
The quantity on the left hand side is the Free Energy change $\Delta \mathcal{F}$ and the work done on the system is the expected cumulative cost given by the right hand side of the equation. Substituting the relevant entities for the MDP defined in Section~\ref{sec:clavmdp} provides a bridge between MDP and the respective thermodynamic interpretation. Therefore, using the mathematical equivalence, the policy that minimizes work done (or maximum work extracted from the system) gives the optimal solution for the MDP.

The above results give sufficient evidence to explore equivalence between thermodynamic entities and Markov Decision Processes. In order to develop a learning framework in uncertain MDPs using information theoretic arguments, I develop the definition of thermodynamic quantities at the level of sample trajectories for Markovian system in the next section.  

\subsection{Second law of thermodynamics for a Markovian system in a heat bath}~\label{sec:sltms}
This section reviews the stochastic thermodynamics for Markovian Systems \citep{ito2016information}. Stochastic Thermodynamics is a theoretical framework to define quantities such as work and heat at the level of sample trajectories. \\


Consider a system $M$ that evolves stochastically. We assume a physical situation where system $M$ is connected to a single heat bath at inverse temperature $\beta$. Also assume that the system $M$ is driven by an external parameter $\pi$ and the system is not subject to non-conservative forces. For simplicity, we will assume discrete time $t_k, k =\{1,2,\cdots,N\}$, although, the mathematical setup does not force any assumption regarding the continuity of time. Let $x_k$ be the state of the system at time $t_k$, and $\pi_k$ be the external parameter of the system at time $t_k$. Let $p(x_k|x_{k-1}, \pi_k)$ be the conditional probability of state $x_k$ under the past trajectory and external parameter $\pi_k$. \\

Now building on the thermodynamic principles, we define the Hamiltonian of the system as $E(x_k,\pi_k)$. The Hamiltonian change in the system is decomposed into 2 parts heat $Q_k$ and work $W_k$. The heat absorbed by the system from heat bath at time $t_k$ is defined as
$$
Q_k = E(x_{k+1},\pi_k) - E(x_{k},\pi_k),
$$
and the work done on the system $M$ is defined as
$$
W_k = E(x_k,\pi_{k}) - E(x_k,\pi_{k-1}).
$$
For a given trajectory $\{x_1,x_2,\cdots,x_n\}$ the total heat is $Q = \sum_{i=1}^{N-1} Q_k$ and total work is $W = \sum_{i=1}^{N-1} W_k$, where $x_0$ is defined as a buffer state such that $p(x_1|x_0,\pi) = 1$ for any $\pi$. This is done to impose consistency as it will become apparent later on.

Using the above definitions, one can easily show that $\Delta E_k = Q_k + W_k$, which is the first law of thermodynamics.\\

Now let us define the quantity $p_B(x_{k}|x_{k+1},\pi_k)$ as the backward transition probability. In the absence of any non conservative the detailed balance \citep{seifert2005entropy} is satisfied which gives

$$
\frac{p(x_{k+1}|x_k,\pi_k)}{p_B(x_k|x_{k+1},\pi_k)} = e^{-\beta Q_k}.
$$

Now, I define the stochastic entropy of the system as 
$h(x_k) = -ln(x_k)$. Therefore the \textit{entropy production} is defined as the sum of stochastic entropy change in the system and the bath. The stochastic entropy change in the system is given by
$$
\Delta h^M_k = h(x_{k+1}) - h(x_k).
$$
The total stochastic entropy change therefore is given by 
\begin{equation}\label{eq:ents}
\Delta h^M = ln\frac{p(x_1)}{p(x_N)}.
\end{equation}
The stochastic entropy change in the heat bath is given by the heat dissipation into the bath
$$
\Delta h_k^{bath} = -\beta Q_k.
$$

The total entropy change in the bath is given by
\begin{equation}\label{eq:enthb}
\Delta h^{bath} = ln\frac{p(x_N|x_{N-1},\pi_{N-1})p(x_{N-1}|x_{N-2},\pi_{N-2})\ldots p(x_2|x_{1},\pi_{1})}{p_B(x_1|x_{2},\pi_{1})p_B(x_{2}|x_{3},\pi_{2})\ldots p(x_{N-1}|x_{N},\pi_{N-1})}.
\end{equation}

Therefore the entropy production $\sigma$ is 
$$
\sigma = ln\frac{p(x_N|x_{N-1},\pi_{N-1})p(x_{N-1}|x_{N-2},\pi_{N-2})\ldots p(x_2|x_{1},\pi_{1}))p(x_1)}{p_B(x_1|x_{2},\pi_{1})p_B(x_{2}|x_{3},\pi_{2})\ldots p(x_{N-1}|x_{N},\pi_{N-1})p(x_N)}.
$$

For brevity I define the trajectory of the system as $O = \{x_1,x_2,\ldots,x_N\}$. Therefore the total entropy production becomes

$$
\sigma = ln\frac{p(O)}{p_B(O)}.
$$
Therefore, the entropy production is determined by the ratio of the probabilities of a trajectory and its time-reversal.

Simple algebraic calculation on this definition yields the second law of thermodynamics which states that

$$
E[\sigma] \geq 0.
$$

The equivalent stochastic energetics definition gives the form as in Equation~(\ref{eq:slt})
$$
W \geq \Delta \mathcal{F},
$$
where $\mathcal{F}(\lambda_k) = -\beta^{-1} \ln \sum_{X} \exp(-\beta E(x,\lambda_k))$. This result can be derived using the integral fluctuation theorem and the arguments presented in~\cite{seifert2005entropy}.

\subsection{Second law of thermodynamics for a Markovian system in connection with an external entity}\label{sec:setup2}
Here I consider the Markovian System $M$ in contact with an external system $D$ in addition to the heat bath. This external system, for instance can be the decision maker in the context of the MDP (More on this in the later sections). In particular, I state the generalized second law of thermodynamics, which states that the entropy production is bounded by the initial and final mutual information between $M$ and $D$, and the transfer entropy from $M$ to $D$.\\

Let's consider the states of the system $D$ at time $t_k$ be $d_k$. Therefore, the joint time evolution of system $M\cup D$ is defined as $\{(x_1,d_0),(x_2,d_1),\cdots,(x_N,d_{N-1})\}$. For brevity, I define $pa(x_{k+1})$ as the parent of state $x_{k+1}$ which is the set of all states which has a non zero transition probability to $x_{k+1}$, therefore $pa(x_{k+1}) = \{x_k,d_{k-1}\}$, such that $p(x_{k+1}|x_k,d_{k-1}) > 0$.\\

At the initial state I assume that $pa(x_1) \subseteq D$. The initial correlation between system $S$ and $D$ is then characterized by the mutual information between $x_1$ and
$pa(s1)$. The corresponding stochastic mutual information is given by 
$$
I_{ini} = I(x_1;pa(x_1)).
$$

Now, let's define $an(x_{k+1})$ as the ancestors of $x_k$ in the order that they were observed. Therefore $an(x_{k+1}) = \{(x_1,d_0),(x_2,d_1),\cdots,(x_k,d_{k-1})\}$. The final correlation between system $S$ and $D$ is then characterized by the mutual information between $x_N$ and $an(x_N) \cap D$. 

$$
I_{fin} = I(x_N;\{d_0,d_1,\cdots,d_{N}\}).
$$

Let's define another quantity $pa(d_k)$as the parent of $d_k$ that corresponds to $pa(d_k) = \{x_{k-1},d_{k-1}\}$Finally, I define the transfer entropy from $M$ to $D$ as

$$
I_{tr}^k = I(d_k;pa(d_k)\cap M|d_1,d_2,\cdots,d_{k-1}).
$$
The total transfer entropy for the entire dynamics is therefore given by 
$$
\sum_{k=1}^N I_{tr}^k = I_{tr}.
$$

By combining all the above informational content in the combined system, I define the total informational exchange as

$$
\Theta = I_{fin} - I_{tr} - I_{ini}.
$$

Now, as in the simple case in Section~\ref{sec:sltms}, I define the entropy production in system $M$ and the heat bath, while in the presence of system $D$.

Let $\mathcal{B}_{k+1} \subseteq D$ define the set of states in $D$ that effect $x_{k+1}$, therefore $\mathcal{B}_{k+1} = \{d_{k-1}\}$. Now  $p(x_{k+1}|x_k, \mathcal{B}_{k+1})$ describes the transition probability from $x_k$ to $x_{k+1}$ under the condition that the states of $D$ that affect $M$ are given by $\mathcal{B}_{k+1}$. We then define the backward transition probability as $P_B(x_k|x_{k+1},\mathcal{B}_{k+1})$. Following the definition of entropy change in the heat bath from time $k$ to $k + 1$ as in Equation~(\ref{eq:enthb})  is given by:
\begin{align*}
\Delta s^{bath} &=\sum_{k} x_k^{bath}\\
&= \sum_{k} \ln\frac{p(x_{k+1}|x_k,B_{k+1})}{p_B(x_k|x_{k+1},B_{k+1})}.
\end{align*}

The total entropy change in the system $M$ is similar to Equation~(\ref{eq:ents})
$$
\Delta s^{sys} = \ln \frac{p(x_1)}{p(x_N)}.
$$

The total entropy production is therefore,
$$
\sigma = \ln\frac{p(x_1)}{p(x_N)}\Pi_{k}\frac{p(x_{k+1}|x_k,B_{k+1})}{p_B(x_k|x_{k+1},B_{k+1})}.
$$

Now, we can write the refined second law of thermodynamics, through some algebraic manipulation it can be shown that
$$
E[\sigma] \geq \Theta.
$$

Using the integral fluctuation theorem and theory of stochastic energetics, this result can be restated as in Equation~(\ref{eq:sltref})

\begin{equation}\label{eq:sltfe}
W - \Delta\mathcal{F} \geq -\beta^{-1}\Theta
\end{equation}

\section{MDP with uncertainty: a stochastic thermodynamics perspective}
The framework in the previous section provides a way to model the effect of information gain in MDPs with uncertainty with the objective of maximizing the work that can be extracted out of the system. The system $M$ considered in the previous section is the system that is acting in the real environment, the system $D$ is the decision maker, who changes some parameter of the system $M$ in order to achieve the required objective. Both these systems are suspended in a ``heat bath" to account for the part of the work that is dissipated and cannot be used for any useful work. The thermodynamic framework allows us to define the objective of the optimization program when the MDPs have model uncertainty. To be consistent, the uncertainty in the MDPs is assumed to be completely reflected through the uncertainty in the transition probabilities. In this section, I propose 2 different perspectives of how the system $D$ interacts with system $M$: a) the first perspective is where system $D$ directly maintains a distribution over the policies and changes this distribution based on feedback; b) the second perspective is where the system $D$ maintains a distribution over a parameter of the transition distribution and adapts this based on feedback in order to find a good policy. \\

\subsection{MDP with distribution over policies}\label{sec:mdpp}
Consider an MDP $M = \{S,A,T,R,N\}$\footnote{Please note that for the purpose of this discussion I will consider $R$ as the cost function (rather than the reward function)} and the decision maker $D = \{\pi\}$. The decision maker maintains a probability distribution $\nu_k(\pi|s_k)$ over policies $\pi$ at every time step $t_k$ in state $s_k$. The probability distribution is updated based on feedback. This setup is analogous to the thermodynamic setup described in Section~\ref{sec:setup2}.
In a standard MDP, the objective is to minimize the expected cumulative cost
$$
V^\pi(s_t) = \sum_{k = t}^{N-1} E[c(s_k,\pi(s_k))],
$$
where the expectation is taken over $\{s_k,\pi(s_k)\}$, in terms of the classical discrete MDP $c(s_k,\pi(s_k)) = E_{s_{k+1} \sim T(\cdot|s_k,\pi(s_k))}[R(s_{k+1}|s_k,\pi(s_k))]$.

As in Section~\ref{sec:clavmdp}, the MDP problem for finding a policy to achieve the maximum \textit{performance} can be formulated as either a maximum entropy optimization program or the classical expected cost optimization. In will start by formulating an expected cost optimization program  using the Second Law of Thermodynamics.  Equation~\ref{eq:sltfe} can be written as
$$
W + \beta^{-1}\Theta \geq \mathcal{F}.
$$
Note that the free energy $\mathcal{F}$ is the amount of useful energy, and the infimum of the left hand side will give the most amount of net work that can be extracted out of the system. Therefore, the optimization program becomes

$$
\min_{\nu_t;t=1:N} W + \beta^{-1}\Theta,
$$
where $W = \sum_{k = 1}^{N-1} E[c(s_k,\pi(s_k))]$, $\Theta = I_{fin} - I_{tr} - I_{ini}$, and $\nu_t = p(\pi_t|s_t,\pi_{t-1})$. From previous section $x_k = \{s_k\}$, 
and $d_k = \pi_k$. For the classical MDP $p(s_1) = \delta(s_1 - s_{init})$, and $pa(s_1) = \emptyset$. Therefore,
$$
I_{ini} = I(s_1;pa(s_1)) = 0.
$$
The final information correlation is given by
$$
I_{fin} = I(s_N;\pi_1,\cdots,\pi_{N-1}),
$$
note that $p(\pi_1,\cdots,\pi_{N-1}) = \Pi_{i=2}^{N-1} p(\pi_{i}|\pi_{i-1})$.
$$
I_{tr}^k = I(\pi_k;s_{k-1}|\pi_1,\cdots,\pi_{k-1}) = I(\pi_k;s_{k-1}|\pi_{k-1})
$$

Therefore the optimization program becomes
$$
\min_{\nu_t:t=1,\cdots,N} \left(\sum_{k = 1}^{N-1} \left( E[c(s_k,\pi(s_k))] - \beta^{-1}I(\pi_{k+1};s_{k}|\pi_{k}) \right)+ \beta^{-1}I(s_N;\pi_1,\cdots,\pi_{N-1})\right). 
$$

For the case, $I_{fin} = 0$, the solution to the resulting optimization program is discussed in \cite{tanaka2017finite}.\\

The above problem can reformulated as a maximum entropy principle, which translates to
$$
\max_{\nu_t:t=1,\cdots,N} \sum_{k = 1}^{N-1} I(\pi_{k+1};s_{k}|\pi_{k}) - I(s_N;\pi_1,\cdots,\pi_{N-1})
$$
subject to
$$
\sum_k \nu_t(\pi_k) = 1\ \forall k,
$$

$$
\sum_{k = 1}^{N-1} E[c(s_k,\pi_k)] = K,
$$
where $K$ is the required performance. When $K = V^*$, the resulting policy is the optimal policy with respect to the cost based optimization program.
\subsection{MDP with parametric uncertainty}\label{sec:mdppu}
In this case the decision maker $D$ maintains a distribution over the parameter of the system. The state of the system $D$ is denoted by $\lambda_k$ at time $t_k$. Again, the specific informational correlations are given by
$$
I_{ini} = I(s_1;pa(s_1)) = 0,
$$
$$
I_{fin} = I(s_N;\lambda_1,\cdots,\lambda_{N-1}),
$$
and
$$
I_{tr}^k = I(\lambda_k;s_{k-1}|\lambda_1,\cdots,\lambda_{k-1}) = I_{tr}^k = I(\lambda_k;s_{k-1}|\lambda_{k-1}).
$$

The optimization program becomes
$$
\min_{\nu_t:t=1,\cdots,N} \left(\sum_{k = 1}^{N-1} \left( E[c(s_k,\pi(s_k))] - \beta^{-1}I(\lambda_{k+1};s_{k}|\lambda_{k}) \right)+ \beta^{-1}I(s_N;\lambda_1,\cdots,\lambda_{N-1})\right), 
$$
where $\nu_t = p(\pi_t|s_t)$, and the distribution over $\lambda$ is updated using Bayesian learning. As in the previous section this can also be formulated as a maximum entropy framework.

\section{Discussion and future work}
This article provides a framework for formulating an optimization program for solving uncertain MDPs built from fundamental principles of system dynamics and information theory. The exact formulation of the optimization program depends on the specific nature of interaction between the decision maker and the system to be controlled. Sections~\ref{sec:mdpp} and \ref{sec:mdppu} provide optimization program for 2 different scenarios. Given these formulation, we can use many of the techniques for optimization (including the Bellman's principle) to solve for a solution. This will be a future work. An important discussion point is the entity $\beta$ in the above equations. Thermodynamically, $\beta$ capture the inverse temperature (with a scaling constant). The temperature is a property of the heat bath and assumed to be constant throughout the dynamic process. In the context of a decision process, the temperature is a property of the decision process and can be estimated. A good way to estimate temperature will be to find an \textit{equilibrium} solution and solve it inversely to get the temperature. For instance, given a MDP $M=\{S,A,T,R,N\}$ one can chose the starting state for which there a solution is known apriori and that can be used to estimate the temperature of the decision process. In the event we do not have access to this knowledge, the temperature can be considered a pseudo state and a new MDP can be defined $M'=\{\{S,\beta\},A,T',R',N\}$.\\

Another important point is that the above framework works when certain conditions on the underlying Markovian process is satisfied. One sufficient condition, as discussed in Section~\ref{sec:msslt}, is the detailed balance equation, which implies reversibility of the Markovian system. We know that is not a necessary condition, in fact, it can be shown that the results still hold for non-reversible Langevin dynamics. Additional research is required to state and prove the necessary and sufficient conditions for this framework to hold.\\

In conclusion, this work opens up avenues for further research in employing information theoretic arguments to learning in MDPs with model uncertainty. This work explicitly models information content and system dynamics for MDPs. I provide a framework to formulate the optimality criterion for MDPs with model uncertainty. Hopefully, future work can extend the rich theory of MDPs to learn and make good decisions in the situations of information uncertainty. 
\clearpage
\begin{appendices}\label{app:maxent}
\textbf{\huge Appendix}
\section{Maximum Entropy Principle: A Control Theoretic Approach}
A classic paper by \cite{saridis1988entropy} derives a maximum entropy framework for Control systems. I present this here as it is interesting to see connections without using the definitions from Thermodynamics.  

Consider a generic decision system formulated in a classic control theoretic framework. Assume that the dynamics of the system are deterministic for simplicity. Then the dynamics are given by
$$
\dot{x}(t) = f(x,u,t),\quad x(t_0) = x_0
$$
and the associated cost function is
$$
V^*(x_0,u,t_0) = \int_{t_0}^T L(x,u,t)dt; \quad L(x,u,t) > 0.
$$
Here, $x(t) \in X$ is a $n$-dimensional state vector and $u(t):X \times T$ is the $m$ dimensional feedback control law. The solution is to find a control law $u_k(x,t)$ such that the value function $V$ will take a value $K$ such that $V_{min}\leq K <\infty$.
$$
V^*(x_0,t_0| u_K(x(t),t) = K
$$
This satisfies the Hamilton-Jacobi-Bellman equation
$$
\frac{\partial V}{\partial t} + \frac{\partial V^T}{\partial x}f(x, u_k,t) + L(x,u_k,t) = 0.
$$

In order to formulate the problem in entropy terms we consider the 
decision-maker's uncertainty of selecting the proper control from the set of admissible 
controls to satisfy the value function requirement to equal $K$ ($V_{min}$ in the 
case of optimal control). This may be expressed as a condition that the expected value of $V$ equals $K$:
$$
E_{u\sim p(u)}[V^*(x_0,t_0;u(x,t)] = K.
$$
The expected value of $V$ is taken over the set of admissible controls $U$, over 
which a probability density $p(u)$ is assumed to express the uncertainty of 
selecting the proper control. The corresponding entropy can then be expressed 
as
$$
H(u,p) = - \int_U p(u) \ln p(u) du.
$$
According to Jaynes principle \citep{jaynes1957information}, the least biased estimate possible on the given information is given by the probability distribution $p(u)$ that maximizes the above entropy $H(u,p)$. Following the method of Lagrange, define 
$$
I = H(u) -\mu (E[V] - K) - \lambda(\int p(u)du -1).
$$
Using calculus of variation to maximize $I$ with respect to the distribution $p(u)$  
yields
$$
\ln(p) + 1 +\mu V + \lambda = 0.
$$
Therefore,
$$
p(u) = \exp^{-1 - \lambda - \mu V^*(x_0,u(t),t)},
$$
and the entropy with maximum information is given by
$$
H(u) = 1+\lambda +\mu E[V^*(x_0,u(x),t)].
$$
For optimality, a control policy $u$ is computed that minimizes the above entropy. This is, therefore, a max-min problem.

\cite{saridis1988entropy} generalizes this analysis in the presence of dynamical uncertainty. Consider that $y\in Y$ is the observation on the state $x$. It is essentially shown that the entropy $H(u)$ can be decomposed in three parts as
$$
H(u) = H(u|y) + H(y) - H(y|u),
$$
where the associated probabilities are given by
\begin{gather}
p(u|y) = e^{-1-\lambda-\mu W(u(y), t)}\\
p(y) = e^{-\rho -\nu\int_0^T ||y-x||^2dt}\\
p(y|u) = p(u|y)p(y)/p(u).
\end{gather}
Here, $W(u(y),t) = E_{x_0,w(t)} \{V^*(x_0,u(x,t),w,\nu,t_0\}$, and $\rho,\nu$ are appropriate constants for the entropy estimation of $H(y)$ based on Jayne's principle.

In case of parametric uncertainty, when
$$
\dot{x} = f(x,u,\lambda,w,t)
$$
when $y$ are the observations
$$
H(y) = H(u|y, \lambda) + H(y|\lambda) + H(\lambda) - H(y,\lambda|u).
$$
An interesting observation is that entropy in a stochastic control system is decoupled into 4 different parts which can be individually computed. 

\end{appendices}

\medskip
\bibliographystyle{abbrvnat}
\bibliography{refer}

\end{document}